\documentclass[12pt,a4paper]{article}
\usepackage[square,sort,comma,numbers]{natbib} 

\usepackage[LGR,T1]{fontenc}
\usepackage[greek,english]{babel}
\usepackage{arabtex}
\usepackage{utf8}
\newcommand\ararab[2][]{{\setcode{utf8}\RL{#2}}}

\usepackage{bbding}

\usepackage{amstext}
\usepackage{amscd}
\usepackage{amsmath}
\usepackage{amssymb}
\usepackage{amsthm}

\newcommand{\C}{\mathbb C}
\newcommand{\N}{\mathbb N}

\renewcommand{\R}{\mathbb R}
\newcommand{\Z}{\mathbb Z}
\newtheorem{theorem}{Theorem}

\newtheorem{lemma}[theorem]{Lemma}

\providecommand{\keywords}[1]
{
  \small	
  \textbf{\textit{Keywords:}} #1
}

\providecommand{\mcs}[1]
{
  \small	
  \textbf{\textit{MCS Subject classification:}} #1
}

\newcommand{\sample}{\ararab{كتاب اَلْمَجِسْطِيّ}} 

\newcommand{\interParagraphSpace}{\vspace{0.5cm}}

\begin{document}

\title{An unfeasability view of neural network learning}
\author{Joos Heintz$^{1}$, Hvara Ocar$^{2}$, Luis Miguel Pardo$^{3}$,\\
Andr\'es Rojas Paredes$^{4}$, Enrique Carlos Segura$^{1}$} 

\footnotetext[0]{Research partially supported by the following Argentinean grant UBACyT 20020190100354BA.}

\footnotetext[1]{Computer Science Department, Facultad de Ciencias Exactas y Naturales, University of Buenos Aires, Pabellón I, Ciudad Universitaria (1428), Ciudad Autónoma de Buenos Aires, Argentina. {\it e--mail}: {\tt \{joos,esegura\}@dc.uba.ar}}

\footnotetext[2]{Facultad de Ingenier\'ia, University of Buenos Aires, Av. Paseo Col\'on 850 (C1063) Ciudad Autónoma de Buenos Aires, Argentina. {\it e--mail}: {\tt jocar@fi.uba.ar}}

\footnotetext[3]{Departamento de Matem\'aticas, Estad\'{\i}stica y Computaci\'on, Facultad de Ciencias, Universidad de Cantabria, 39071
Santander, Spain. {\it e--mail}: {\tt luis.m.pardo@gmail.com}}

\footnotetext[4]{Instituto de Ciencias, Universidad Nacional General Sarmiento, J. M. Guti\'errez 1150 (B1613GSX) Los Polvorines, Provincia de Buenos Aires, Argentina. {\it e--mail}: {\tt arojas@campus.ungs.edu.ar}}

\maketitle

\begin{abstract}
We define the notion of a continuously differentiable perfect learning algorithm for multilayer neural network architectures and show that such algorithms don't exist provided that the length of the data set exceeds the number of involved parameters and the activation functions are logistic, tanh or sin.

\interParagraphSpace 

\interParagraphSpace
\mcs{68T07, 68Q32}

\end{abstract}




\section{Introduction}

%
%

The present paper deals with the question whether (supervised) learning by means of neural networks based on the usual activation functions logistic, tanh and sinusoid has limited applications or is an universal tool. For this purpose we introduce an idealization of the notion of learning algorithm for multilayer neural networks. This idealization is inspired in the backpropagation procedure and is called perfect learning algorithm. It relies on a specification $\Pi$, also called perfect, which assigns to each training data a parameter vector which constitutes a global minimizer of the quadratic error function involved, if the error reaches an exact minimum. A perfect learning algorithm has a perfect specification $\Pi$ and assigns to each training data set a numerical representation of a parameter vector which satisfies the specification $\Pi$. The existing versions of backpropagation become then interpreted as attempts
to satisfy the requirement of perfectness of learning algorithms. Therefore it depends on our definition of algorithm whether we are able to affirm that perfect learning algorithms really exist. We may also simply ask whether there exist continuously differentiable perfect learning algorithms where the specification $\Pi$ is continuously differentiable.

The aim of this paper is to give a negative answer to this question in case of differentiable perfect learning algorithms, provided the length of the training data set exceeds the number of involved parameters and the activation functions are logistic, tanh or sin.

\interParagraphSpace

Automatic learning is in fact not a modern concept, it has a long history. In particular, ancient
greek astronomy was marked by a deep epistemological discussion about the scope of
automatic learning techniques. Since the beginning of ancient greek astronomy in the 4th
century BC by Eudoxus of Cnidus and Callippus of Cyzicus there was a tacit assumption that the explanation and prediction of the motions of heavenly bodies and the description of their
nature must rely on geometrical models. This leads to the question of the principles (and their
foundations) which should be satisfied in advance by the geometric model which is used for
the description of the celestial phenomena. Nevertheless, answering this question was considered
as a subject of physics, whereas genuine astronomy became restricted to the exact description
of the orbits of the heavenly bodies. Only in a final stage of this reasoning the parameters of
the geometric model under consideration should be adjusted by the astronomer in order to
make the model predictive. This had to be done primarily in order ``to save the appearances''
(\foreignlanguage{greek}{σῴζειν τὰ φαινόμενα}), and only secondarily with the aim to justify the particular geometric
model. This modus operandi anticipated already the modern concept of automatic learning.
Nevertheless, let us observe that this --for our historical view of automatic learning well
suited-- exposition of interaction between distinct epistemological concepts didn't remain
undisputed between specialists of ancient greek astronomy. Anyway, the restrictions
introduced by physics to the geometric modelling of astronomy (mainly under the influence of
Plato and Aristotle in the 5th and 4th century BC) led to a long--lasting effort with successively
changing geometric models in order to improve their explanatory and predictive power. The
geometric models inspired by the physics of Plato and Aristotle were geocentric and divided
the cosmos in two regions, the motionless spherical earth and a heavenly region surrounding
it, containing multiple spheres rotating at different speeds around distinct axes. In this sense,
Eudoxus of Cnidus, motivated by Plato's requirement that the planetarian orbits should be
decomposable into uniform circular motions, presupposed for their geometric models that the
heavenly bodies move each one with his own constant angular speed on concentric circles
around the motionless center of the earth. Each one of these bodies moved thus on the
equator of the corresponding sphere.

In the sequel the development of ancient greek astronomy was characterized by the
introduction of a series of new concepts and views into the geometric modelling of the
motions of heavenly bodies in order to improve the explanatory and predictive power for
observable celestial phenomena. Simultaneously the requirement was maintained that the
physical principle of the uniform circular motions of the heavenly bodies should be preserved.
This became achieved by admitting uniform epicyclic and uniform excentric circular motions
(on a circle, called deferent, with center distant from the earth).


This eventful development of geocentric astronomy converged finally in the 2th century AD to the mathematical and astronomical treatise of Claudius Ptolemy on the apparent motions of
the stars and planetary paths, today remembered as “Almagest” (from arabic \sample \ ), whereas the original title was ``\foreignlanguage{greek}{Μαθηματικὴ Σύνταξις}''. In order to achieve the goal to account for the observed motions of planets conserving the principle of uniform circular movements,
Ptolemy introduced a final mathematical tool, the so called equant point, with respect to
which the epicycle under consideration moves with constant angular speed along the deferent
of the excentricity.

Ptolemy's work influenced as of the abbasid period strongly the astronomy, first of the muslim
and then of the medieval christian world. Even Copernicus who popularized in the 16th
century AD the heliocentric point of view for the planetary system maintained still one of the
standard beliefs of his time, namely that the motions of celestial bodies must be decomposable into uniform circular movements. This point of view obliged him to retain for his orbit calculations a complex system of epicycles like in the Ptolemaic system.

In this context let us remark that the geocentric point of view was dominant in ancient greek
astronomy but not exclusive. In the 3rd century BC Aristarchus of Samos proposed a
heliocentric geometric model for astronomical tasks which allowed him to estimate in terms of
earth radii the sizes of the sun and the moon as well as their distances from the earth.
However, these estimations turned later out to be far below the real ones.

\interParagraphSpace

Like in the case of automatic learning by neural networks all these astronomical calculations
produce only approximative results. Nevertheless, one may ask whether the intricate orbits of
all heavenly bodies really may be arbitrarily well approximated in this simple geometric way,
piling up, if necessary, sufficiently many epicycles and epicycles of epicycles, etc.

An analogous question may be asked also for automatic learning by neural networks with
given activation functions. Certain answers to this question are the subject of the so called
universal approximation theorems. For example, an early one of these theorems says that
standard multilayer feedforward networks with as few as one hidden layer using arbitrary
activation functions are capable to approximate up to any desired degree of accuracy any
Borel measurable function from one finite dimensional space to another, provided sufficiently
many hidden units are available. In this sense, the multilayer feedforward networks form a
class of universal approximators \cite{HSW89}.

Turning back to the first question about the motions of heavenly bodies from the geocentric
point of view, an answer was given in \cite{S25} by the italian astronomer Giovanni
Schiaparelli (1835--1910). In the case of a single heavenly body moving in the plane around the
origin, the piling up of an arbitrary finite number of epicycles such that each one moves with
constant angular velocity, gives rise to a finite sum of trigonometric functions, which
represents a kind of generalized Fourier polynomial. The polynomials obtained in this way
approximate under a suitable seminorm arbitrary well any Besicovitch almost periodic
function.

\interParagraphSpace

For details about origins and development and subsequent influence of ancient greek
astronomy we refer to \cite{D08} and \cite{D13}.

\newpage 

\section{An unfeasibility result}
Let $m,n\in\N$ and $X_l$, $V_k$, $S$, $W_l^k$, $T_k$, $1 \leq k \leq m$, $1 \leq l \leq n$ indeterminates and $X:=(X_1,\dots,X_n)$, $V:=(V_1,\dots,V_m)$, $W:=(W_l^k)_{ \substack{1 \leq k \leq m \\ 1 \leq l \leq n}}$, $T:=(T_1,\dots,T_m)$.

Let $f$ and $g_1,\dots,g_m$ be suitable on $\R$ defined activation functions and $g:\R^m \to \R^m$ the map defined by $g(u_1,\dots,u_m):=(g_1(u_1),\dots,g_m(u_m))$ with $(u_1,\dots,u_m)\in\R^m$. From now on we shall deal only with three layer neural networks with inputs $X_1,\dots,X_n$, one single output, $m$ neurons on the hidden layer, $m(n+1)$ weights, $m+1$ thresholds and activation functions $f$ and $g_1,\dots,g_m$. Formally we can describe the architecture of these networks by
\begin{align*} 
O_{V,S,W,T}(X) := & f(S+V \cdot g(T + W \cdot X)) \\ 
                = & f(S+\sum_{1\leq k \leq m} V_k g_k (T_k + \sum_{1\leq l \leq n} W^k_l X_l)),
\end{align*}
where the dot refers to the inner vector and the matrix--vector products. 

Let $p\in\N$ and
\[
\mathcal{U}_p:=\{ ((\gamma_1, \zeta_1),\dots,(\gamma_p, \zeta_p)) \in\R^{p\times (n+1)}, \gamma_1,\dots,\gamma_p \text{\ all\ distinct } \}. 
\]
The elements of $\mathcal{U}_p$ constitute the training data of length $p$ which we are going to consider.

Let $(\gamma_1,\dots,\gamma_p)\in \R^{p\times n}$ a sequence of $p$ distinct points of $\R^n$. For each parameter vector $(v,s,w,t)\in\R^m\times\R\times\R^{m\times n}\times \R^m = \R^{m(n+2)+1}$ the architecture $O_{V,S,W,T}(X)$ produces a neural network 
\[ o_{v,s,w,t}(X):= f(s+v\cdot g(t+w\cdot X))
\]
which can be evaluated in $\gamma_1,\dots,\gamma_p$ returning the vector 
\[
((\gamma_1, o_{v,s,w,t}(\gamma_1)),\dots,(\gamma_p, o_{v,s,w,t}(\gamma_p))) \in\mathcal{U}_p.
\]

The task of a perfect learning algorithm is to find, if possible, for any training example $((\gamma_1, \zeta_1),\dots,(\gamma_p, \zeta_p))\in\mathcal{U}_p $ a parameter vector $(v,s,w,t)\in \R^{m(n+2)+1}$ such that the quadratic error function 
\[
E(v,s,w,t):=\sum_{1\leq i \leq p} (\zeta_i - o_{v,s,w,t} (\gamma_i))^2 
\] 
reaches a global minimum exactly.

If such a minimum does not exist, no condition is imposed on $(v,s,w,t)$. Thus we may specify a perfect learning algorithm by a map $\Pi:\mathcal{U}_p\to\R^{m(n+2)+1}$ which assigns to each training data set an exact global minimizer of the error function if such a minimizer exists. We shall call the perfect algorithm continuosly differentiable if its specification $\Pi$ it is.

In the sequel let $A_1,\dots,A_p$ and $B_1,\dots,B_p$ new indeterminates.

With these notations we may formulate the following result.

\begin{lemma}
\label{lemma perfect algorithm}  
Let $f$ and $g_1,\dots,g_m$ be continuously differentiable with $f'(0)\neq 0$ and let $O_{V,S,W,T}(X)=f(S+V\cdot g(T+W\cdot X))$ the neural network architecture considered before. Suppose that the generic determinant $\text{det}(g_1(A_i B_j))_{1\leq i,j \leq p}$ does not vanish identically. Then for $p> m(n+2)+1$ there does not exist a continuosly differentiable perfect algorithm satisfying the specification $\Pi:\mathcal{U}_p \to \R^{m(n+2)+1}$ above.
\end{lemma}

\begin{proof}
By assumption we have $\text{det}(g_1(A_i B_j))_{1\leq i,j \leq p} \neq 0$. One concludes easily that we may choose $\rho_1,\dots,\rho_p,\gamma_1,\dots,\gamma_p\in\R^n$ with $\gamma_1,\dots,\gamma_p$ all distinct such that $\text{det}(g(\rho_i \cdot \gamma_j))_{1\leq i,j \leq p}\neq 0$ holds. Let $\theta:\R^{m(n+2)+1}\to\R^p$ the map which assigns to each parameter vector $(v,s,w,t)\in \R^{m(n+2)+1}$ the image $\theta(v,s,w,t):= (o_{v,s,w,t}(\gamma_1),\dots,o_{v,s,w,t}(\gamma_p))$ and observe that $\theta$ is continuously differentiable. Suppose that the statement of the lemma is wrong. Then there exists nonnegative integer parameters $p$, $m$, $n$ with $p>m(n+2)+1$ and continuously differentiable perfect algorithm satisfying the specification above.

Let $\pi:\R^p\to\R^{m(n+2)+1}$ be the map which assigns to $\zeta=(\zeta_1,\dots,\zeta_p)\in\R^p$ the image $\pi(\zeta):= \Pi((\gamma_1,\zeta_1),\dots,(\gamma_p,\zeta_p))$. The specification $\Pi$ is continuously differentiable thus same holds also true for $\pi$. Since $\Pi$ specifies a perfect learning algorithm we see that for each parameter vector  $(v,s,w,t)\in \R^{m(n+2)+1}$ the parameter vector $(\pi\circ\theta)(v,s,w,t)) $ minimizes the quadratic error function 
\[
E(v',s',w',t'):=\sum_{1\leq i \leq p} (o_{v,s,w,t}(\gamma_i) - o_{v',s',w',t'} (\gamma_i))^2 
\] 
for $(v',s',w',t')\in \R^{m(n+2)+1}$.

This minimum is zero. Thus we have $o_{v,s,w,t}(\gamma_i) = o(\pi\circ\theta)(v,s,w,t)$ for $1\leq i \leq p$ and therefore the identity
\begin{equation}
\label{theta}
\theta\circ\pi\circ\theta=\theta
\end{equation}

For $1\leq i \leq p$ let $v_i:= (V_1,0,\dots,0)$ $s_i:=0$ $w_i\in\R^{m\times n}$ the matrix which contains $\rho_i$ as its first row and $0$ elsewhere and $t:=(0,\dots,0)$. Let $\beta_i(V_1):= (v_i,s_i,w_i,t_i)$. The corresponding function $\beta_i:\R\to\R^{m(n+2)+1}$ is continuously differentiable. We have
\begin{align*} 
(\theta\circ\beta_i)(V_1)  & = (o_{v_i,s_i,w_i,t_i}(\gamma_1),\dots,o_{v_i,s_i,w_i,t_i}(\gamma_p))  \\
& = (f(V_1 g_1(\rho_i \cdot \gamma_1)),\dots,f(V_1 g_1(\rho_i \cdot \gamma_p)))
\end{align*}
and therefore 
\[
\frac{d}{d V_1}(\theta \circ \beta_i)(0)= f'(0)(g_1(\rho_i\cdot \gamma_1),\dots,g_1(\rho_i\cdot \gamma_p)).
\]

Since $\pi$ is continuosly differentiable it is in particular differentiable in $\theta\circ\beta_i(0)=(f(0),\dots,f(0))$ and therefore $\pi\circ\theta\circ\beta_i$ is differentiable in $0$. We infer from the chain rule applied to the identity \eqref{theta}
\begin{align*} 
f'(0)(g_1(\rho_i\cdot\gamma_1),\dots,g_1(\rho_i\cdot\gamma_p)) & = \frac{d}{d V_1} (\theta \circ \beta_i)(0)  \\ 
             & = d \theta((\theta\circ\beta_i)(0)) (\frac{d}{d V_1} (\pi\circ\theta \circ \beta_i)(0)).
\end{align*}

Observe now that $(\theta\circ\beta_i)(0)=(f(0),\dots,f(0))$ is independent from $1\leq i \leq p$ and therefore also the linear map $M:=d\theta(\pi\circ\theta\circ\beta_i)(0)$. Recalling that $f'(0)\neq 0$ and $\text{det}g_1((\rho_i\cdot\gamma_j))_{1\leq i,j \leq p} \neq 0$ holds we see that the $p$ vectors $f'(0)(g_1(\rho_i\cdot\gamma_1),\dots,g_1(\rho_i\cdot\gamma_p)), 1\leq i \leq p$, are all linearly independent.

On the other and the vectors $\frac{d}{d V_1} (\pi\circ\theta \circ \beta_i)(0), 1\leq i \leq p $, are all contained in $\R^{m(n+2)+1}$ and mapped by the linear map $M$ on the previous $p$ vectors. This implies $p \leq m(m+2)+1$, which contradicts our assumption $p > m(m+2)+1 $. 
\end{proof}

It seems us worth to comment that our proof of Lemma \ref{lemma perfect algorithm} requires the differentiability of $\Pi$ only in one single point, namely $(\gamma_1, f(0), \dots, \gamma_p, f(0))$.  

\interParagraphSpace 
Let $g:\R\to\R$ be a function of class $\C^{\infty}$ satisfying an algebro--differential equation as follows: there exists a polynomial $G\in\R[T]$ of positive degree with $g'=G(g)$ (here $T$ is a new indeterminate). Suppose that $G(g(0))\neq 0$ holds. Let $p\in\N$ and $A_1,\dots,A_p,B_1\dots,B_p$ indeterminates as before. 

With these notations and assumptions we have
\begin{lemma}
\label{lemma determinant g}
  $\text{det}(g(A_iB_j))_{1\leq i,j \leq p}\neq 0$ 
\end{lemma}
\begin{proof}
By induction on $p$. For $p=1$ there is nothing to prove, because $g'(0)=G(g(0))\neq 0$ implies $g\neq 0$.

Developing the determinant $\text{det}(g(A_iB_j))_{1\leq i,j \leq p}$ in the first row we obtain polynomials $M_1,\dots,M_p$ in $g(A_iB_j),1\leq i \leq p, 1\leq j \leq p$ such that 
\[
\text{det}(g(A_iB_j))_{1\leq i,j \leq p} = \sum_{1\leq i \leq p} g(A_iB_1)M_i 
\]
holds. 

Since $(-1)^{i+1}M_i$ is a cofactor of the matrix $g(A_iB_j)_{1\leq i,j \leq p}$ we may apply our inductive hypothesis and conclude $M_1\neq 0,\dots,M_p\neq 0$. For $k\in\Z_{\geq 0}$ let $P_k\in\R[T]$ be recursively defined as $P_0:=T$ and $P_{k+1}:=P_k'G$. Then we have $\text{deg}P_{k+1} = \text{deg}P_k'+\text{deg}G \geq \text{deg}P_k$ and therefore $P_k$   is of positive degree for all $k\in\N$. In particular we conclude $P_k \neq 0$ for any $k\in\Z_{\geq 0}$. Observe that 
\begin{align*} 
\frac{\partial}{\partial B_1} \sum_{1\leq i \leq p} P_k(g(A_iB_1))A_i^k M_i & = \sum_{1\leq i \leq p} P_k'(g(A_iB_1))g'(A_iB_1))A_i^{k+1} M_i = \\ 
\sum_{1\leq i \leq p} (P_k'\cdot G)(g(A_i\cdot B_1)) A_i^{k+1}M_i &  = \sum_{1\leq i \leq p} P_{k+1}(g(A_i\cdot B_1))A_i^{k+1}M_i
\end{align*}
and 
\[
\text{det}g(A_i\cdot B_j)_{1\leq i \leq j} = \sum_{1\leq i \leq p} g(A_iB_1)M_i = \sum_{1\leq i \leq p} P_0(g(A_iB_j))\cdot A_i^0 M_i 
\]
holds. This implies 
\begin{equation}
\label{estrella}
\frac{\partial^k}{\partial B_1^k} \text{det}(g(A_i\cdot B_j))_{1\leq i,j \leq p} = \sum_{1\leq i \leq p} P_k(g(A_i\cdot B_1)) A_i^k M_i
\end{equation}

Since $P_k\neq 0$ holds we can write $P_k = Q_k (T - g(0))^{m_k}$ for some polynomial $Q_k\in\R[T]$ with $Q_k(g(0))\neq 0$ and $m_k\in\Z_{\geq 0}$. Observe that $P_{k+m_k}$ can be written as $P_{k+m_k}= S(T) (T-g(0))+ Q_k G^{m_k}$, where $S(T) \in\R[T]$ is a suitable polynomial. Therefore we have $P_{k+m_k}(g(0))\neq 0$. This implies that for any $k\in\Z_{\geq 0}$ there exists $k'\geq k$ with $P_{k'}(g(0))\neq 0$. Now we may chose integers $0 \leq k_1 < \dots < k_p$ such that $P_{k_j}(g(0))\neq 0$ holds for $1\leq j \leq p$.

Assume now that  $\text{det}(g(A_iB_j))_{1\leq i,j \leq p}$ vanishes identically. Then \eqref{estrella} implies 
\[
\sum_{1\leq i \leq p } P_k (g(A_iB_1)) A_i^k M_i = 0
\]
for any $k\in\Z$ and in particular 
\[
\sum_{1\leq i \leq p } P_{k_j} (g(A_iB_1)) A_i^{k_j} M_i = 0 \text{\ \ and\ } \sum_{1\leq i \leq p } P_{k_j} (g(0)) A_i^{k_j} M_i = 0
\]
for any $1\leq j \leq p$. 

Observing $\text{det}(P_{k_j}(0)A_i^{k_j})_{1\leq i,j \leq p}\neq 0 $ we conclude $M_1=0,\dots,M_p=0$ which is a contradiction.  
\end{proof} 

\begin{lemma}
\label{lemma determinant}
  Let notations be as in the previous lemma. Then we have
\[
\text{det}(\text{sin}A_iB_j)_{1\leq i,j \leq p} \neq 0.
\]
\end{lemma}
\begin{proof}
Again by induction in $p$. The case $p=1$ is obvious. In order to treat the case $p>1$ we develop the determinant $\text{det}(\text{sin}A_iB_j)_{1\leq i,j \leq p}$ as before obtaining polynomials $M_1,\dots,M_p$ in $\text{sin}(A_iB_j)$, $1\leq i \leq p$, $1\leq j \leq p$ such that
\[
\text{det}(\text{sin}A_iB_j)_{1\leq i,j \leq p} = \sum_{1\leq i \leq p} \text{sin}(A_iB_1)M_i
\] 
holds. For $1\leq i \leq p$ the expression $(-1)^{i+1}M_i$ is a cofactor of $(\text{sin} A_iB_j)_{1\leq i,j \leq p}$ and we can conclude inductively $M_1\neq 0,\dots, M_p\neq 0$.

We assume $\text{det}(\text{sin}A_iB_j)_{1\leq i,j \leq p} = 0$ and derive iteratively the identity
\[
\frac{\partial^k}{\partial B_1^k} \sum_{1\leq i \leq p} \text{sin} (A_i B_1) M_i = 0
\]
with respect to $B_1$. This yields for $k\in\Z_{\geq 0}$ identities 
\[
\sum_{1\leq i \leq p} \text{cos} (A_i B_1) A_i^{2k+1}M_i = 0
\]
and therefore 
\[
\sum_{1\leq i \leq p} A_i^{2k+1} M_i =0.
\]
Similarly as before we conclude $M_1=0,\dots,M_p=0$, a contradiction.
\end{proof} 

Taking into account the well--known first order differential equations for the logistic and tanh functions, we may summarize the outcome of Lemma \ref{lemma perfect algorithm}, \ref{lemma determinant g}, \ref{lemma determinant} by the following statement which demonstrates the unfeasibility of continuously differentiable perfect learning algorithms.

\begin{theorem}
  \label{theorem 1}
  For the usual activation functions logistic, tanh and sinusoid there do not exist differentiable perfect learning algorithms able to learn for any neural network architecture with at least one hidden layer any training data of length exceeding the number of parameters. 
\end{theorem}

Let us observe that we do not dispose over a general algorithmic model able to capture in its whole extent the notion of learning algorithm. Hence, in view of Theorem \ref{theorem 1}, we are only able to state a kind of metaconjecture, namely that general perfect learning algorithms do not exist. The practical counterpart of this can also be found by practical experience. In particular there is no option to improve backpropagation to a perfect learning algorithm. This argument may be reinforced experimentally as follows. 

Let $p>>m(n+2)+1$ and choose $v,s,w,t$ and $\gamma_1,\dots,\gamma_p$ at random (so $\gamma_1,\dots,\gamma_p$ are all distinct). Compute numerical representations for $\zeta_1 := o_{v,s,w,t}(\gamma_1),\dots, \zeta_p := o_{v,s,w,t}(\gamma_p)$ and apply backpropagation to this representation of the training data set $((\gamma_1,\zeta_1),\dots,(\gamma_p,\zeta_p))$. 

The algorithm returns a numerical representation of a parameter vector $(v',s',w',t')$ and an error $E(v',s',w',t')$. We may expect that 
\[
E(v',s',w',t')>>0
\]
holds. This situation is made possible by local minima. In consequence the usual justification of backpropagation with reference to global minima is incomplete. Therefore the real fundamentation of backpropagation is exclusively based on practical evidence, not on theory.

Finally let us state that mutatis mutandis Theorem \ref{theorem 1} is also true for neural network architectures over the complex numbers with polynomial activation functions. This can easily be seen combining the arguments of the proof of \cite[Theorem 18]{BHM16} with the arguments of the proof of Lemma \ref{lemma perfect algorithm}.

\section{Outlook}

In textbooks backpropagation becomes usually motivated as an attempt to solve by a simple
algorithm a particular global minimization problem. Theorem \ref{theorem 1} expresses under moderate
differentiability conditions the unfeasability of this purpose.

In case of neural networks with activation functions which are polynomials over the reals (or,
more generally, semialgebraic functions), efficient real quantifier elimination procedures may
be applied in order to solve the corresponding global minimization problem (see \cite{HRS90}, \cite{R} and in
particular \cite[14.2]{BPR}). This way to proceed leads to complexity upper bounds which are singly
exponental in the number of parameters of the neural network architecture under
consideration. Nevertheless, in this general setting we cannot expect more efficient worst case
complexity bounds.

On the other hand, the activation functions logistic, tanh and sin are pfaffian. Hence, for
learning purposes, it makes sense to consider more generally neural network architectures
with arbitrary pfaffian activation functions. In order to proceed in an analogous way as in the
polynomial or semialgebraic case for the solution of the underliying global minimization
problem, the subpfaffian set up is the most suitable one. The final complexity outcome is
similar as in the polynomial or semialgebraic case (see \cite{GV01} and the survey \cite{GV04}).

\bibliographystyle{alphaabbr}
\bibliography{references}

\end{document}